\documentclass[twoside,11pt]{article}

\usepackage{jmlr2e}
\usepackage{xcolor}
\usepackage{amsmath}
\usepackage{todonotes}

\newtheorem{assumption}[theorem]{Assumption}

\newcommand{\EE}{\mathbb{E}}

\ShortHeadings{}{}
\firstpageno{1}

\begin{document}

\title{ Correction to \textit{``Wasserstein distance estimates for the distributions of numerical
approximations to ergodic stochastic differential equations"} }

\author{\name Daniel Paulin\hspace{0.3mm}\footnotemark[1] \email dpaulin@ed.ac.uk\\
       \addr School of Mathematics\\
       University of Edinburgh, United Kingdom 
       \AND
    \name Peter A. Whalley\hspace{0.3mm}\thanks{Both authors contributed equally.} \email peter.whalley@math.ethz.ch\\
       \addr Seminar for Statistics\\
       ETH Z{\"u}rich, Z{\"u}rich, Switzerland 
       }

\editor{}

\maketitle

\begin{abstract}
A method for analyzing non-asymptotic guarantees of numerical discretizations of ergodic SDEs in Wasserstein-2 distance is presented by Sanz-Serna and Zygalakis in \textit{Wasserstein distance estimates for the distributions of numerical
approximations to ergodic stochastic differential equations}. They analyze the UBU integrator which is strong order two and only requires one gradient evaluation per step, resulting in desirable non-asymptotic guarantees, in particular $\mathcal{O}(d^{1/4}\epsilon^{-1/2})$ steps to reach a distance of $\epsilon > 0$ in Wasserstein-2 distance away from the target distribution. However, there is a mistake in the local error estimates in \cite{Sanz-Serna2021}, in particular, a stronger assumption is needed to achieve these complexity estimates. This note reconciles the theory with the dimension dependence observed in practice in many applications of interest.
\end{abstract}

\section{Introduction}
In \cite{Sanz-Serna2021}, the authors present a framework to analyze the convergence rate and asymptotic bias in Wasserstein-2 distance of numerical approximations to ergodic SDEs. In their framework, they consider underdamped Langevin dynamics on $\mathbb{R}^{2d}$ which is given by 
\begin{equation} \label{eq:ULD}
    \begin{split}
        dv &= -\gamma v dt - c\nabla f(x)dt + \sqrt{2\gamma c}dW(t)\\
        dx &= vdt,
    \end{split}
\end{equation}
with $c,\gamma >0$. It can be shown under mild assumptions that \eqref{eq:ULD} is ergodic and has invariant measure $\pi^{*}$ with density proportional to $\exp{\left(-f(x) - \frac{1}{2c}\|v\|^{2}\right)}$. 

\subsection{Assumptions}
Let $\mathcal{H}:\mathbb{R}^{d} \to \mathbb{R}^{d \times d}$ be the Hessian of $f$.
\begin{assumption}\label{assum:hessian}
    $f:\mathbb{R}^{d} \to \mathbb{R}^{d}$ is twice differentiable, $m$-strongly convex and $L-$smooth, i.e.
    \[
    \forall x \in \mathbb{R}^{d}, \qquad m I_{d \times d} \prec \mathcal{H}(x) \prec L I_{d \times d},
    \]
    where $I_{d \times d} \in \mathbb{R}^{d \times d}$ is the $d$-dimensional identity matrix.
\end{assumption}

\begin{assumption}\label{assum:higher_order_smoothness}
    Let $\mathcal{H}:\mathbb{R}^{d} \to \mathbb{R}^{d \times d}$ be the Hessian of $f$, $f$ be three times differentiable and there is a constant $L_{1} \geq 0$ such that at each point $x \in \mathbb{R}^{d}$, for arbitrary $(w_{1},w_{2}) \in \mathbb{R}^{d}\times \mathbb{R}^{d}$
    \[
    \|\mathcal{H}'(x)[w_{1},w_{2}]\| \leq L_{1}\|w_{1}\| \|w_{2}\|.
    \]
\end{assumption}
Note that this assumption can be also reformulated in terms of the following norm.
\begin{definition}\label{def:123norm}
    For $A \in \mathbb{R}^{d \times d \times d}$
    \[
    \|A\|_{\{1\}\{2\}\{3\}} = \sup\left\{\left.\sum^{d}_{i,j,k=1}A_{ijk}x_{i}y_{j}z_{k} \right| \sum_{i}x_i^2\leq 1,\sum_{j}y_j^2\le 1, \sum_{k}z_k^2\le 1\right\}
    \]
\end{definition}
Then it is easy to check that Assumption \ref{assum:higher_order_smoothness} is equivalent to 
    \begin{equation}\label{eq:assum1equivalent}
    \|\mathcal{H}'(x)\|_{\{1\}\{2\}\{3\}} \leq L_{1}\quad \text{ for every }x \in \mathbb{R}^{d}.
    \end{equation}
\subsection{Wasserstein distance}
\begin{definition}
    Let $\pi_{1}$ and $\pi_{2}$ be two probability measures on $\mathbb{R}^{2d}$, the $2$-Wasserstein distance between $\pi_{1}$ and $\pi_{2}$ with respect to the positive definite matrix $P \in \mathbb{R}^{2d \times 2d}$ is given by
    \begin{equation}
        \mathcal{W}_{P}(\pi_{1},\pi_{2}) = 
        \left(\inf_{\zeta \in Z}\int_{\mathbb{R}^{2d}} \|x-y\|^{2}_{P} d\zeta (x,y)\right)^{1/2},
    \end{equation}
    where $Z$ is the set of all couplings between $\pi_{1}$ and $\pi_{2}$ and $\|\xi\|_{P}  = \left(\xi^{T}P\xi\right)^{1/2}$ for all $\xi \in \mathbb{R}^{2d}$.
\end{definition}

\section{UBU integrator and error in dimension dependence}
In \cite{Sanz-Serna2021} they considered the UBU integrator, a numerical integrator for \eqref{eq:ULD}, which is strong order $2$ and only requires one gradient evaluation per iteration. It is defined by the following updated rule for a stepsize $h>0$
\begin{align}
    v_{n+1} &= \mathcal{E}(h)v_{n} - h\mathcal{E}(h/2)c\nabla f(y_{n}) + \sqrt{2\gamma c}\int^{t_{n+1}}_{t_{n}}\mathcal{E}(t_{n+1}-s)dW(s),\label{eq:UBU_v}\\
    x_{n+1} &= x_{n} + \mathcal{F}(h)v_{n} -h\mathcal{F}(h/2)c\nabla f(y_{n})+ \sqrt{2\gamma c}\int^{t_{n+1}}_{t_{n}}\mathcal{F}(t_{n+1}-s)dW(s),\label{eq:UBU_x}\\
    y_{n} &= x_{n} + \mathcal{F}(h/2)v_{n} + \sqrt{2\gamma c}\int^{t_{n+1/2}}_{t_{n}}\mathcal{F}(t_{n+1/2}-s)dW(s),\label{eq:UBU_y}
\end{align}
where $\mathcal{E}(t) = \exp{\left(-\gamma t\right)}$ and $\mathcal{F}(t) = \frac{1-\exp{\left(-\gamma t\right)}}{\gamma}$. Due to the high strong order properties of the scheme, the UBU scheme shows improved dimension dependence in terms of non-asymptotic guarantees, $\mathcal{O}(d^{1/4})$. This is supported by numerics in some applications in \cite{zapatero2017word,chada2023unbiased}. In \cite{Sanz-Serna2021} under Assumptions \ref{assum:hessian} and \ref{assum:higher_order_smoothness} they show in Theorem 25 this improved dimension dependence, however, there is a mistake in the local error estimates, in particular on page 31 they make use of the bound
\begin{equation}\label{eq:dimension_inequality}
    \mathbb{E}\left(\|\mathcal{H}'(x)[\overline{v},\overline{v}]\|^{2}\right) \leq L^{2}_{1}\mathbb{E}\left(\|\overline{v}\|^{4}\right) \leq 3L^{2}_{1}c^{2}d
\end{equation}
for $\overline{v} \sim \mathcal{N}(0,c)$. However, it is straightforward to show that 
\[
\mathbb{E}\left(\|\overline{v}\|^{4}\right) = c^{2}(d^{2}+2d)
\]
and under Assumption \ref{assum:higher_order_smoothness} it is not possible to achieve non-asymptotic guarantees of order $\mathcal{O}(d^{1/4})$. The focus of this note is to reconcile the numerics with the theory by introducing a stronger assumption than Assumption \ref{assum:higher_order_smoothness}, used in \cite{chen2023does} which achieves non-asymptotic guarantees of order $\mathcal{O}(d^{1/4})$ and is verifiable for many applications. This additional assumption is introduced in Section \ref{sec:strongly_hessian_lipschitz}. We also correct an issue in the fifth step of the local error proof in \cite{Sanz-Serna2021} and correct the constant $C_{0}$ in Theorem 25. This is done in Section \ref{sec:C_0new} and in Section \ref{sec:Theorem25} a new version of Theorem 25 is stated with the new constants.

\subsection{Strongly Hessian Lipschitz assumption}\label{sec:strongly_hessian_lipschitz}
\cite{chen2023does} introduced the notion of strongly Hessian Lipschitz. We are going to use this new concept (Assumption \ref{assum:strong_hessian_lip}) instead of Assumption \ref{assum:higher_order_smoothness}, which was used in \cite{Sanz-Serna2021}, to correct dimension dependence in inequality \eqref{eq:dimension_inequality}. 

The strongly Hessian Lipschitz property relies on the following tensor norm.
\begin{definition}
    For $A \in \mathbb{R}^{d \times d \times d}$, let
    \[
    \|A\|_{\{1,2\}\{3\}} = \sup_{x\in \mathbb{R}^{d\times d}, y\in \mathbb{R}^d}\left\{\left.\sum^{d}_{i,j,k=1}A_{ijk}x_{ij}y_{k} \right| \sum_{i,j=1}^{d}x^{2}_{ij}\leq 1, \sum^{d}_{k=1}y^{2}_{k}\leq 1\right\}
    \]
\end{definition}
\begin{assumption}\label{assum:strong_hessian_lip}
    $f:\mathbb{R}^{d} \to \mathbb{R}$ is thrice differentiable and $L^{s}_{1}$-strongly Hessian Lipschitz if 
    \[
    \|\mathcal{H}'(x)\|_{\{1,2\}\{3\}} \leq L^{s}_{1}
    \]
    for all $x \in \mathbb{R}^{d}$.
\end{assumption}

From the definition, it is clear that $\|A\|_{\{1,2\}\{3\}}\ge \|A\|_{\{1\}\{2\}\{3\}}$, and so by \eqref{eq:assum1equivalent}, the strong Hessian Lipschitz property (Assumption \ref{assum:strong_hessian_lip}) implies the Hessian Lipschitz property (Assumption \ref{assum:higher_order_smoothness}) with $L_1=L_1^{s}$. We will show a result in the other direction in Lemma \ref{lem:123bnd}.

It is also easy to check that the strong Hessian Lipschitz property does not introduce dimension dependency for product target distributions, since $A\in \mathbb{R}^{d\times d\times d}$ with  $A_{ijk}=0$ unless $i=j=k$ (diagonal tensors), $\|A\|_{\{1,2\},\{3\}}=\max_{i} |A_{iii}|$, and so $L_1^s$ equals the maximum of the supremum of the absolute value of third derivatives amongst the potentials of all components. Other examples of interest which do not introduce dimension dependency include the Bayesian multinomial regression (see \cite[Lemma H.6]{chada2023unbiased}), ridge separable functions and $2$-layer neural networks \cite{chen2023does}.

The following Lemma allows us to control the quantity 
\[\EE_{p\sim \mathcal{N}(0,I_d)}\left(\|A[p,p,\cdot]\|^2\right)=\EE_{p\sim \mathcal{N}(0,I_d)}\left[\sum_{k\le d}  \left(\sum_{i,j\le d} A_{ijk} p_i p_j\right)\left(\sum_{l,m\le d} A_{lmk} p_l p_m\right)\right].\]
This is a special case of Lemma 13 of \cite{chen2023does}, but with an explicit constant in the bound (their constant was not made explicit).
\begin{lemma}\label{lemma:Gaussian_chaos}
Let $p\sim \mathcal{N}(0,I_d)$, and \[g(p)=\sum_{k\le d}  \left(\sum_{i,j\le d} A_{ijk} p_i p_j\right)\left(\sum_{l,m\le d} A_{lmk} p_l p_m\right).\]
Then \[\EE g(p)\le 3 d \|A\|_{\{12\}\{3\}}^2.\]
\end{lemma}
\begin{proof}
Due to the independence of the components,  $\EE(p_{i}p_{j}p_{l}p_{m})=0$ unless there are either two pairs of the same or four of the same amongst the indices.
We know that $\EE(p_i^4)=3$. Hence, we have
\begin{align*}
\EE(g(p))=\sum_{i,j,k} \left(A_{iik}A_{jjk} + A_{ijk}^2+A_{ijk}A_{jik}\right)\le
\sum_{i,j,k} \left(A_{iik}A_{jjk} +2 A_{ijk}^2\right).
\end{align*}
We have that 
\begin{align*}
&\|A\|_{\{12\}\{3\}}=\sup_{x,y}\left\{\left.\sum_{i_1,i_2,i_3} A_{i_1i_2i_3}x_{i_1i_2}y_{i_3}\right|\sum_{i_1,i_2} x_{i_1i_2}^2\le 1, \sum_{i_3} y_{i_3}^2\le 1\right\}\\
&=\sup_{y} \left\{\left.\left(\sum_{i_1,i_2} \left(\sum_{i_3}A_{i_1i_2i_3} y_{i_3}\right)^2\right)^{1/2}\right| \sum_{i_3} y_{i_3}^2\le 1 \right\}\\
&=\sup_{y:\|y\|\le 1} \left(\sum_{i_1,i_2} \left<A_{i_1,i_2,\cdot}, y\right>^2\right)^{1/2}=\left\|\sum_{i_1,i_2} A_{i_1,i_2,\cdot}\cdot A_{i_1,i_2,\cdot}^T\right\|^{1/2}\\
&=\left\|\sum_{i_1} A_{i_1,\cdot,\cdot}^T\cdot A_{i_1,\cdot,\cdot}\right\|^{1/2}.
\end{align*}
Now, it is easy to see that
\[2\sum_{i,j,k} A_{ijk}^2= 2\mathrm{Tr}\left(\sum_{i_1}A_{i_1,\cdot,\cdot}^T\cdot A_{i_1,\cdot,\cdot}\right)\le 2d \|A\|_{\{12\}\{3\}}^2.\]
For the other term $\sum_{i,j,k} A_{iik}A_{jjk}$, we define a matrix $\overline{A}\in \mathbb{R}^{d\times d}$ as $\overline{A}_{ik}=A_{iik}$. Let $e\in \mathbb{R}^{d}$ be a vector of ones, i.e. $e_1=1,\ldots,e_d=1$.  Using these, we have
\begin{align*}
\sum_{i,j,k} A_{iik}A_{jjk}=e^T \overline{A}^T \overline{A} e
\le d\|\overline{A}\|^2.
\end{align*}
We have
\begin{align*}
&\|A\|_{\{12\}\{3\}}=\sup_{x,y}\left\{\left.\sum_{i_1,i_2,i_3} A_{i_1i_2i_3}x_{i_1i_2}y_{i_3}\right|\sum_{i_1,i_2} x_{i_1i_2}^2\le 1, \sum_{i_3} y_{i_3}^2\le 1\right\}\\
\intertext{by only considering $x$ such that $x_{i_1i_2} = 0$ for $i_1 \neq i_2$,}
&\ge \sup_{x,y}\left\{\left.\sum_{i_1,i_3} A_{i_1i_1i_3}x_{i_1i_1}y_{i_3}\right|\sum_{i_1} x_{i_1i_1}^2\le 1, \sum_{i_3} y_{i_3}^2\le 1\right\}=\|\overline{A}\|,\end{align*} hence $\sum_{i,j,k} A_{iik}A_{jjk}\le d \|A\|_{\{12\}\{3\}}^2$, and the claim of the lemma follows.
\end{proof}

Using Lemma \ref{lemma:Gaussian_chaos} and starting from the top of page 31 we have 
\[
\mathbb{E}(\|\mathcal{H}'(\overline{x})[\overline{v},\overline{v}]\|^{2}) \leq 3(L^{s}_{1})^{2}c^{2}d,
\]
and we can simply replace $L_{1}$ by $L^{s}_{1}$ in the estimates of \cite{Sanz-Serna2021}.

Note that although the strongly Hessian Lipschitz assumption is stronger than the previous Hessian Lipschitz assumption (i.e. $L_{1}\le L_{1}^{s}$), it is possible to show that every Hessian Lipschitz function is also strongly Hessian Lipschitz, due to the following result.

\begin{lemma}\label{lem:123bnd}
For any $A\in \mathbb{R}^{d\times d\times d}$, $\|A\|_{\{12\}\{3\}}\le \sqrt{d} \|A\|_{\{1\}\{2\}\{3\}}$. Hence every $L_1$-Hessian Lipschitz function is $\sqrt{d}L_1$-strongly Hessian Lipschitz.
\end{lemma}
\begin{proof}
In the proof of Lemma \ref{lemma:Gaussian_chaos}, we have shown that
\begin{align*}
\|A\|_{\{12\}\{3\}}=\left\|\sum_{i_1} A_{i_1,\cdot,\cdot}^T\cdot A_{i_1,\cdot,\cdot}\right\|^{1/2}.
\end{align*}
Let $e_i=(0,\ldots, 0,1,0,\ldots,0)\in 
\mathbb{R}^d$ be the unit vector with 1 in component $i$. Using Definition \ref{def:123norm}, we have for every $i\le d$,
\begin{align*}&\|A\|_{\{1\}\{2\}\{3\}} = \sup\left\{\left.\sum^{d}_{i,j,k=1}A_{ijk}x_{i}y_{j}z_{k} \right| \sum_{i}x_i^2\leq 1,\sum_{j}y_j^2\le 1, \sum_{k}z_k^2\le 1\right\} \\
&\ge \sup\left\{\left.\sum^{d}_{i,j,k=1}A_{ijk}x_{i}y_{j}z_{k} \right| x_i=e_i,\sum_{j}y_j^2\le 1, \sum_{k}z_k^2\le 1\right\}
=\|A_{i,\cdot,\cdot}^T\cdot A_{i\cdot,\cdot}\|^{1/2},\end{align*}
and the claim follows by rearrangement.
\end{proof}

\section{Non-asymptotic guarantees for UBU}
In \cite{Sanz-Serna2021} they consider contraction and discretization bias of the UBU scheme with $\gamma = 2$ and $h \leq 2$ in the $2$-Wasserstein distance with respect to $P_{h} := \hat{P} \otimes I_{d}$, where
\begin{equation}\label{eq:P}
    \hat{P} = \begin{pmatrix}
        1 & 1\\
        1 & 2
    \end{pmatrix}.
\end{equation}
They define the weighted Hilbert-space norm with respect to the matrix $P_{h}$, $\|\cdot \|_{L^{2},P_{h}}$ and corresponding inner product $\left\langle \cdot,\cdot \right\rangle_{L^{2},P_{h}}$.
They then define the iterates of UBU, $(\xi_{n})_{n \in \mathbb{N}}$ with $\xi_{0} \sim \pi$, by the update rule $(v_{n+1},x_{n+1}) = \xi_{n+1} = \psi_{h}(\xi_{n},t_{n})$, $t_{n} = nh$, $n = 0,1,...$, $h >0$ to be the one-step approximation governed by the UBU integrator with initial condition $\xi \in \mathbb{R}^{2d}$ and $\phi_{h}(\cdot,\cdot)$ to be the exact counterpart with shared Brownian incremements. They introduce at each time level $n$, the random variable $\hat{\xi}_{n}\sim \pi^{*}$ to be the optimal coupling such that
$\mathcal{W}_{P}(\pi^{*},\Psi_{h,n}\pi) = \|\hat{\xi}_{n}-\xi_{n}\|_{L^{2},P_{h}}$, where $\Psi_{h,n}\pi$ is the measure of $\xi_{n}$. They have shown in \cite{Sanz-Serna2021}[Theorem 18] that for $\gamma = 2$, $c=\overline{c}/(L+m)$, where $\overline{c} \in (0,4)$, there is a $h_{0}$, such that, for any $h \leq h_{0}$ and $n \in \mathbb{N}$
\begin{equation}\label{eq:contraction}
    \|\xi^{(2)}_{n+1} -\xi^{(1)}_{n+1}\|^{2}_{P_{h}} \leq \rho_{h}\|\xi^{(2)}_{n} -\xi^{(1)}_{n}\|^{2}_{P_{h}},
\end{equation}
where $\rho_{h} \in (0,1)$ for realizations $(\xi^{(i)}_{k})_{k \in \mathbb{N}}$ for $i =1,2$ of the UBU discretization.

\begin{assumption}[Assumption 22 of \cite{Sanz-Serna2021}]\label{assum:22}
There is a decomposition
\[
\phi_{h}(\hat{\xi}_{n},t_{n})-\psi_{h}(\hat{\xi}_{n},t_{n}) = \alpha_{h}(\hat{\xi}_{n},t_{n}) + \beta_{h}(\hat{\xi}_{n},t_{n}),
\]
and positive constants $p$, $h_{0}$, $C_{0}$, $C_{1}$, $C_{2}$ such that for $n \geq 0$ and $h \leq h_{0}$:
\begin{equation}
    \left|\left\langle \Psi_{h}(\hat{\xi}_{n},t_{n}) - \psi_{h}(\xi_{n},t_{n}),\alpha_{h}(\hat{\xi}_{n},t_{n})\right\rangle \right| \leq C_{0}h\|\hat{\xi}_{n}-\xi_{n}\|_{L^{2},P_{h}}\|\alpha_{h}(\hat{\xi}_{n},t_{n})\|_{L^{2},P_{h}}
\end{equation}
and 
\begin{equation}
    \|\alpha_{h}(\hat{\xi}_{n},t_{n})\|_{L^{2},P_{h}}\leq C_{1}h^{p+1/2}, \qquad \|\beta_{h}(\hat{\xi}_{n},t_{n})\|_{L^{2},P_{h}}\leq C_{2}h^{p+1}.
\end{equation}
\end{assumption}

 \begin{theorem}[Theorem 23  of \cite{Sanz-Serna2021}]
     Assume that the integrator satisfies \cite{Sanz-Serna2021}[Assumption 22] and in addition, there are constants $h_{0}>0$, $r>0$ such that for $h\leq h_{0}$ the contractivity estimate \eqref{eq:contraction} holds with $\rho_{h} \leq (1-rh)^{2}$. Then, for any initial distribution $\pi$, stepsize $h \leq h_{0}$, and $n=0,1,...,$
     \begin{equation}
         \mathcal{W}_{P_{h}}(\pi^{*},\Psi_{h,n}\pi) \leq (1-hR_{h})^{n}\mathcal{W}_{P_{h}}(\pi^{*},\pi) + \left(\frac{\sqrt{2}C_{1}}{\sqrt{R_{h}}} + \frac{C_{2}}{R_{h}}\right)h^{p}
     \end{equation}
     with
     \[
     R_{h} = \frac{1}{h}\left(1-\sqrt{(1-rh)^{2} + C_{0}h^{2}}\right) = r + o(1), \qquad \textnormal{as} \qquad h \downarrow 0.
     \]
 \end{theorem}
 The remaining subsections are devoted to correcting the constants in Assumption 22 for the UBU integrator and $p=2$ and thereby the non-asymptotic guarantees by Theorem 23 of \cite{Sanz-Serna2021}.
 \subsection{The local error of UBU: Error in the fifth step}\label{sec:C_0new}
Throughout this section we use the same notation as \cite{Sanz-Serna2021}[Section 7.6], in particular, $(v_{n+1},x_{n+1}) = \xi_{n+1}$ and $(\Tilde{v}_{n+1},\Tilde{x}_{n+1})$ denotes the velocity component and position of a UBU step initialized at
$\hat{\xi}_n = (\hat{v}_n, \hat{x}_n )$. In the fifth step of \cite{Sanz-Serna2021}[Section 7.6] they use the equality 
\[
\left|\left\langle \psi_{h}(\hat{\xi}_{n},t_{n}) - \psi_{h}(\xi_{n},t_{n}),\alpha_{h}(\hat{\xi}_{n},t_{n})\right\rangle_{L^{2},P_{h}}\right| = |\mathbb{E}(\left\langle \tilde{v}_{n+1}-v_{n+1},\alpha_{v}\right\rangle)|.
\]
Unfortunately, this is incorrect, due to the matrix inner product defined by $P_{h}$, we have
\[
\left|\left\langle \psi_{h}(\hat{\xi}_{n},t_{n}) - \psi_{h}(\xi_{n},t_{n}),\alpha_{h}(\hat{\xi}_{n},t_{n})\right\rangle_{L^{2},P_{h}}\right| = |\mathbb{E}(\left\langle \tilde{v}_{n+1}-v_{n+1},\alpha_{v}\right\rangle + \left\langle \tilde{x}_{n+1}-x_{n+1},\alpha_{v}\right\rangle)|.
\]
The first term on the right-hand-side of this expression was bounded in \cite{Sanz-Serna2021}[Eq. (43)] in terms of $\mathbb{E}(\|\hat{v}_{n}-v_{n}\|^{2})$, $\mathbb{E}(\|\hat{x}_{n}-x_{n}\|^{2})$ and $\mathbb{E}(\|\alpha_{v}\|^{2})$. The additional term can be treated by the same argument as the fourth step of \cite{Sanz-Serna2021}[Section 7.6] that is we estimate
\begin{align*}
    |\mathbb{E}\left(\langle \tilde{x}_{n+1} - x_{n+1},\alpha_{v} \rangle\right)| &= |\mathbb{E}\left(\langle \tilde{x}_{n+1} - \hat{x}_{n} - x_{n+1} + x_{n},\alpha_{v} \rangle\right)|\\
    &\leq \left(\mathbb{E}\| \tilde{x}_{n+1} - \hat{x}_{n} - x_{n+1} + x_{n}\|^{2}\right)^{1/2}\left(\mathbb{E}(\|\alpha_{v}\|^{2})\right)^{1/2}.
\end{align*}
Now, from \eqref{eq:UBU_x},
\[
\tilde{x}_{n+1} - \hat{x}_{n} - x_{n+1} + x_{n} = \mathcal{F}(h)(\hat{v}_{n} - v_{n}) - h\mathcal{F}(h/2)c(\nabla f(\tilde{y}_{n}) - \nabla f(y_{n}))
\]
with \eqref{eq:UBU_y}
\[
\tilde{y}_{n} = \hat{x}_{n} + \mathcal{F}(h/2)\hat{v}_{n} + \sqrt{2\gamma c}\int^{t_{n+1/2}}_{t_{n}}\mathcal{F}(t_{n+1/2}-s)dW(s),
\]
and thus, since $\mathcal{F}(h) \leq h$
\begin{align*}
&\mathbb{E}\left(\|\tilde{x}_{n+1}-\hat{x}_{n} - x_{n+1} + x_{n}\|^{2}\right)^{1/2} \leq h\mathbb{E}\left(\|\hat{v}_{n} - v_{n}\|^{2}\right)^{1/2} + \frac{h^{2}cL}{2}\left(\mathbb{E}\left(\|\tilde{y}_{n}-y_{n}\|\right)^{2}\right)^{1/2}.
\end{align*}
Taking into account \eqref{eq:UBU_y} and the definition of $\tilde{y}_{n}$
\[
(\mathbb{E}\|\Tilde{y}_{n}-y_{n}\|^{2})^{1/2} \leq \left(\mathbb{E}\left(\|\hat{x}_{n}-x_{n}\|^{2}\right)\right)^{1/2} + \frac{h}{2}\left(\mathbb{E}\left(\|\hat{v}_{n}-v_{n}\|^{2}\right)\right)^{1/2}
\]
and we conclude that $|\mathbb{E}\left(\langle \tilde{x}_{n+1}-x_{n+1},\alpha_{v}\rangle\right)|$ is bounded above by
\[
h\left(\frac{hcL}{2}\left(\mathbb{E}\left(\|\hat{x}_{n}-x_{n}\|^{2}\right)\right)^{1/2}+ \left(1+\frac{h^{2}cL}{4}\right)\left(\mathbb{E}\left(\|\hat{v}_{n}-v_{n}\|^{2}\right)\right)^{1/2}\right)\left(\mathbb{E}(\|\alpha_{v}\|^{2})\right)^{1/2}.
\]

\subsection{Theorem 25 of \cite{Sanz-Serna2021}}\label{sec:Theorem25}

In this section we are stating the corrected version of \cite{Sanz-Serna2021}[Theorem 25].
\begin{theorem}
    Assume that $f$ satisfies Assumptions \ref{assum:hessian} and \ref{assum:strong_hessian_lip}. Set $\gamma = 2$ and $P_{h} = \hat{P} \otimes I_{d}$. Then, for $h \leq 2$, the UBU discretization satisfies Assumption 22 of \cite{Sanz-Serna2021}  (Assumption \ref{assum:22}) with $p = 2$,
    \begin{align*}
    C_{0} &= K_{0}\left(3+2cL\right) \\
    C_{1} &= K_{1}c^{3/2}Ld^{1/2}\\
    C_{2} &= K_{2}\left((1+4\sqrt{3})c^{2}L^{3/2} + (3 + \frac{\sqrt{42}}{2})c^{3/2}L + 6cL^{1/2} + \sqrt{3}c^{2}L^{s}_{1}\right)d^{1/2},
\end{align*}
where $K_{j}$, $j = 0,1,2,$ are the following absolute constants
\[
K_{0} = \sqrt{\frac{4}{3-\sqrt{5}}}, \qquad K_{1} = \frac{\sqrt{3}}{12}, \qquad K_{2} = \frac{1}{24}\sqrt{\frac{3 + \sqrt{5}}{2}}.
\]
\end{theorem}
\begin{proof}
Combining estimates of Section \ref{sec:C_0new} in the norm $\|(v,x)\|^{2}_{P_{h}} := \|v\|^{2} + 2\left\langle v,x\right\rangle + 2\|x\|^{2}$ and using the fact that $\|(x,v)\| \leq \sqrt{\frac{2}{3-\sqrt{5}}}\|(x,v)\|_{P_{h}}$ we have
\begin{align*}
 &\left|\left\langle \psi_{h}(\hat{\xi}_{n},t_{n}) - \psi_{h}(\xi_{n},t_{n}),\alpha_{h}(\hat{\xi}_{n},t_{n})\right\rangle_{L^{2},P_{h}}\right|\leq \\
 & h\left(\left(\frac{hcL}{2} + cL\right)\left(\mathbb{E}\left(\|\hat{x}_{n}-x_{n}\|^{2}\right)\right)^{1/2}+ \left(3+\frac{h^{2}cL}{4}+\frac{hcL}{2}\right)\left(\mathbb{E}\left(\|\hat{v}_{n}-v_{n}\|^{2}\right)\right)^{1/2}\right)\left(\mathbb{E}(\|\alpha_{v}\|^{2})\right)^{1/2}\\
 &\leq h\sqrt{\frac{4}{3-\sqrt{5}}}\left(3+2cL\right)\|\hat{\xi}_{n}-\xi_{n}\|_{L^{2},P_{h}}\|\alpha_{h}\|_{L^{2},P_{h}},
\end{align*}
and the required $C_{0}$ constant, with $C_{2}$ given by following the argument of \cite{Sanz-Serna2021}[Section 7.6] using Lemma \ref{lemma:Gaussian_chaos}.
\end{proof}

\acks{We would like to thank Jesus Sanz-Serna and Kostas Zygalakis for the helpful correspondence regarding their paper. }

\vskip 0.2in
\bibliography{sample}

\end{document}